\documentclass[letterpaper, conference]{ieeeconf}
\IEEEoverridecommandlockouts
\usepackage{cite}

\def\BibTeX{{\rm B\kern-.05em{\sc i\kern-.025em b}\kern-.08em
    T\kern-.1667em\lower.7ex\hbox{E}\kern-.125emX}}
    
  \usepackage{amsmath,amssymb,amsfonts,amsthm} 
\usepackage{mathtools}
\usepackage{commath}
\usepackage{algorithmic}
\usepackage[ruled,vlined]{algorithm2e}
\usepackage{graphicx}
\usepackage{textcomp}
\usepackage[table]{xcolor}

\usepackage{framed}
\definecolor{shadecolor}{rgb}{.9,.9,.9}
\usepackage{tcolorbox}
\usepackage{siunitx}
\usepackage{textgreek}
\usepackage{url}
\usepackage{booktabs, tabularx}
\usepackage{multirow}
\usepackage{stackengine}
\usepackage[caption=false, font=footnotesize]{subfig}

\usepackage{tikz}
\usetikzlibrary{quotes,angles}
\usetikzlibrary{calc}
\usetikzlibrary{decorations.pathmorphing}

\usepackage{soul}
\usepackage{todonotes}

\usepackage[absolute]{textpos}     
\usepackage[hidelinks]{hyperref}   

\renewcommand{\vec}[1]{\boldsymbol{\mathbf{#1}}}

\newcommand{\real}{\mathbb{R}}
\def \x {\vec{x}}
\def \fct {V_f}
\def \dis {V_d}
\def \env {\mathbb{W}}
\def \seg {\overline}
\def \ray {\overrightarrow}
\def \bdr {\mathcal{B}}
\def \star {\mathcal{S}}
\def \orb {\mathbb{C}}
\def \cir {\mathcal{C}}
\def \traj {\mathcal{T}}
\def \radius {r}
\def \control {u}
\def \angpos {\Theta}

\usepackage{soul}

\theoremstyle{plain}
\newtheorem{theorem}{Theorem}

\title{\LARGE\bf Flow-Based Control of Marine Robots in Gyre-Like Environments
}

\author{Gedaliah Knizhnik$^1$, Peihan Li$^1$, Xi Yu$^2$, and M. Ani Hsieh$^1$
\thanks{This work was supported by the National Science Foundation Award Nos. CMMI-1760369 and IIS-1812319.}
\thanks{$^1$ Gedaliah Knizhnik, Peihan Li, and M. Ani Hsieh are with the GRASP Laboratory, University of Pensylvannia, Philadelphia, PA 19104. 
        {\tt\footnotesize knizhnik@seas.upenn.edu}}%
\thanks{$^2$ Xi Yu is with the Department of Mechanical and Aerospace Engineering, West Virginia University, Morgantown, WV 26506. 
        {\tt\footnotesize xi.yu1@mail.wvu.edu}}%
}

\newcommand{\copyrightstatement}{
    \begin{textblock*}{5.7in}(0.25in,0.25in) 

        \noindent
        \footnotesize
        This accepted article to ICRA is made available by the authors in compliance with IEEE policy.

        \noindent
        Please find the final, published version in IEEE Xplore, DOI: \href{https://doi.org/10.1109/ICRA46639.2022.9812331}{\textcolor{blue}{10.1109/ICRA46639.2022.9812331}}.
        
    \end{textblock*}

    \begin{textblock*}{5.7in}[0,1](0.25in,10.85in) 

        \noindent
        \footnotesize \scriptsize
        \copyright 2022 IEEE. Personal use of this material is permitted.
        Permission from IEEE must be obtained for all other uses, in any current or future media, including reprinting/republishing this material for advertising or promotional purposes, creating new collective works, for resale or redistribution to servers or lists, or reuse of any copyrighted component of this work in other works.
    \end{textblock*}
}

\begin{document}
\bstctlcite{MyBSTcontrol} 
\copyrightstatement                    

\maketitle

\begin{abstract}
We present a flow-based control strategy that enables resource-constrained marine robots to patrol gyre-like flow environments on an orbital trajectory with a periodicity in a given range. The controller does not require a detailed model of the flow field and relies only on the robot's location relative to the center of the gyre. Instead of precisely tracking a pre-defined trajectory, the robots are tasked to stay in between two bounding trajectories with known periodicity. Furthermore, the proposed strategy leverages the surrounding flow field to minimize control effort. We prove that the proposed strategy enables robots to cycle in the flow satisfying the desired periodicity requirements. Our method is tested and validated both in simulation and in experiments using a low-cost, underactuated, surface swimming robot, \textit{i.e.} the Modboat.
\end{abstract}


\section{Introduction} \label{sec:intro}
Autonomous marine vehicles (AMVs) have been employed to study and understand various biological \cite{ref:Caron2008,ref:Smith2010a}, chemical \cite{ref:Hajieghrary2017}, and physical processes in the ocean \cite{ref:Fiorelli2006}. Processes at the air-sea interface are also important to monitor, as exchanges between the ocean and atmosphere directly impact regional rainfall patterns, storm tracks, and sea levels. Similar to at-depth persistent monitoring, monitoring at the air-sea interface requires AMVs to operate for long periods of time, have the ability to move in and out of different monitoring regions of interest, and operate in high inertia environments whose dynamics are nonlinear and stochastic. 

In this work, we focus on the development of flow-based control strategies for minimally actuated autonomous surface mobile sensors, {\it i.e. active drifters} \cite{ref:Molchanov2014}, for persistent monitoring in dynamic environments like the ocean.  Specifically, we address the design of control strategies for resource-constrained, pointable thruster robots like the Modboat -- a small, low-cost, underactuated surface swimming robot\cite{modboatsOnline,Knizhnik2021ThrustRobot}, that would enable it to stay in and/or move between regions of interest.  Swarms of active drifters like the Modboat can simultaneously collect data at many distinct geographic locations, which is important for wave-height reconstruction and estimating spatiotemporal variations in sea-surface temperatures. Modboats can also adapt, albeit in a limited fashion, their sampling strategies to maximize information gain. Nevertheless, compared to larger, more capable autonomous surface vehicles, Modboats have limited power budgets and must rely on energy aware motion control and coordination strategies.

Existing small scale AMVs, such as AMOUR~\cite{Vasilescu2010AMOURPayloads}, $\mu$AUV~\cite{Hanff2016aUV2Vehicle}, and ANGELS~\cite{Mintchev2014} are designed to maintain powerful actuation so as to be able to overpower currents when accomplishing tasks. Energy efficient monitoring AMVs tend to be gliders~\cite{Leonard2007,Leonard2010,Paley2008CooperativeSystem}, whose actuation mechanism requires them to move in a sawtooth pattern for energy efficiency but makes exploiting surrounding currents challenging. In contrast, there is increasing work focused on planning of time and energy optimal paths for AMVs that leverage environmental flows~\cite{ref:Garau2005,ref:Kruger2007,ref:Witt2008,ref:Koay2013,ref:Rao2009,ref:Chakrabarty2013,ref:Eichhorn2015,ref:Otte2016,ref:Subramani2016,Kularatne2018OptimalDiscretization}. These approaches for planning with flows include graph-search  \cite{ref:Garau2005,ref:Koay2013, ref:Rao2009,ref:Chakrabarty2013,ref:Eichhorn2015,ref:Otte2016,Kularatne2018GoingFlows}, iterative minimization \cite{ref:Kruger2007,ref:Witt2008}, and level set expansion \cite{ref:Subramani2016}.  However, these approaches require full knowledge of the flow field, and obtaining good ocean current forecasts and nowcasts can be difficult. While there are existing strategies that enable planning in the presence of ocean current forecast uncertainties~\cite{ref:Wolf2010,ref:Pereira2013,ref:AlSabban2013,ref:Huynh2015,Kularatne2018OptimalUncertainties}, these approaches still require complete knowledge of the flow.

More recently, approaches have been proposed for resource constrained mobile sensors that rely on predictions of the average time required to switch between adjacent flow gyres in an ocean environment~\cite{ref:Heckman2014JDSM,ref:Heckman2016,Kularatne2018}. In~\cite{ref:Heckman2014JDSM,ref:Heckman2016}, it was shown that the addition of limited controls can enhance or abate the switching times between gyres. In contrast,~\cite{Kularatne2018} presented a strategy that uses limited control to achieve a desired average transition time. Similar to~\cite{ref:Heckman2014JDSM,ref:Heckman2016,Kularatne2018}, we propose a control strategy for minimally actuated resource-constrained mobile sensors that leverages the surrounding flow dynamics. Different from these existing works, we focus on strategies that keep robots in their designated monitoring regions.  Specifically, we present a control strategy that enables resource constrained mobile robots to continuously monitor an environment with gyre-like flow by approximately orbiting around the region's center. As such, the novelty of the contribution lies in the synthesis of a control strategy for mobile sensors that leverages the surrounding flow dynamics without full knowledge of the flow field.

The rest of the paper is organized as follows: gyre-like environments are defined in Sec.~\ref{sec:problem}, and \textit{flow-based control} is presented in Sec.~~\ref{sec:methodology}. Sec.~\ref{sec:expPlatform} details the Modboat as well as flow generation and modeling in the experimental facility. Results of both simulation and experimental validation are presented in Sec.~\ref{sec:results} and discussed in Sec.~\ref{sec:discussion}. Conclusions and directions for future work are discussed in Sec.~\ref{sec:conclusion}.


\section{Problem Formulation}\label{sec:problem}


\subsection{Gyre-Like Flow Fields}

We consider the following $2$-D flow field $\env \subset \real^2$ 
\begin{align} \label{eq:dissDyn}
    \dot{\x} = \fct (\x) + \dis (\x),
\end{align}
where $\x \in \env$. We assume the flow has dissipative dynamics captured by $\dis: \env \mapsto \real^2$ with the non-dissipative part described by $\fct: \env \mapsto \real^2$.  The solutions or trajectories of~\eqref{eq:dissDyn} form roughly concentric orbits. 

While many flows can be described by~\eqref{eq:dissDyn}, the wind-driven double-gyre flow model is a common example and is often used to describe large scale ocean circulation~\cite{veronis1966wind1,veronis1966wind2}. Eq.~\eqref{eq:exp1} is an example of a simplified version of the model where $\mu$ is the dissipation parameter.
\begin{align}\label{eq:exp1}
    \mathbf{Ex. 1}: \begin{pmatrix}
    \dot{x}_1 \\
    \dot{x}_2
    \end{pmatrix} = \begin{pmatrix}
    -\pi A\sin(\frac{\pi x_1}{s})\cos(\frac{\pi x_2}{s})-\mu x_1\\
    \hphantom{-}\pi A\cos(\frac{\pi x_1}{s})\sin(\frac{\pi x_2}{s})-\mu x_2
    \end{pmatrix}
\end{align}

Another example model is a vortex flow that is generated by a spinning blade at a fixed underwater location. The 2-D model of the surface flow can be written as~\eqref{eq:exp2}, where $r$ is the distance between $(x_1,x_2)^T$ and the origin, $\Omega(r)$ the angular velocity satisfying $\Omega(r_1)<\Omega(r_2)$ if and only if (iff.) $r_1>r_2$, and $\mu$ a parameter specifying the dynamics along the radial direction. 
\begin{align}\label{eq:exp2}
    \mathbf{Ex. 2}: \begin{pmatrix}
    \dot{x}_1\\
    \dot{x}_2
    \end{pmatrix} =\begin{pmatrix}
    -\Omega(r)x_2 - \mu x_1\\
    \hphantom{-}\Omega(r)x_1-\mu x_2
    \end{pmatrix},
\end{align}


\subsection{Non-dissipative dynamics of gyre-like flow fields} \label{sec:nonDissipative}

We briefly summarize some important properties of $\fct(\x)$ and $\dis(\x)$.  Consider a flow field with $\dis (\x) =0$ given by 
\begin{align}\label{eq:nonDissDyn}
    \env_f: \dot{\x} = \fct (\x).
\end{align}
The trajectory starting at any $\x(0) =\x_i \in \env_f$ in this flow will form a closed loop. We denote this closed-loop trajectory as $\bdr_i$ and note that $\bdr_i$ is the boundary of a star-shaped domain $\star_i$. The trajectories in $\env_f$ are all non-intersecting, concentric, closed loops that bound star-shaped domains. In fact, the trajectories of the non-dissipated flow field~\eqref{eq:nonDissDyn} are: 
\begin{itemize}
    \item \textbf{Periodic: }If $\x(t) =\x_i \in \env_f$, $\exists~ T_i>0$ such that $\x(t+T_i) = \x(t) = \x_i, \forall t>0, \forall \x_i \in \env_f$.
    \item \textbf{Centered: }$\exists \x_o \in \env_f$, such that if $\x(t) = \x_o$ for some $t>0$, then $\x(t+\delta) = \x_o,$ $\dot{\x}(t+\delta) = 0$, $\forall \delta>0$. 
    \item \textbf{Star-shaped: }$\forall \x_i \in \env_f$, let $\seg{\x_o\x_i}$ be the line segment linking $\x_o$ and $\x_i$, and $\star_i$ be the star-shaped domain bounded by the trajectory $\bdr_i$ passing through $x_i$, we have $\seg{\x_o\x_i} \subset \star_i$. 
    \item \textbf{Non-intersecting: }$\forall \x_i \in \env_f$, let $\ray{\x_o\x_i(0)}$ be the line that starts at $\x_o$ and passes $\x_i(0)$, then $\x_i(t) \not\in \ray{\x_o\x_i(0)}$, except for $t = 0, T_i, 2T_i,...$.
\end{itemize}

Furthermore, consider two trajectories $\bdr_i$ and $\bdr_j$ in the flow field $\env_f$ with the related domains $\star_j\subset \star_i$. We have the following property:
\begin{itemize}
    \item \textbf{Monotonically changing angular velocities} For any line $\ray{\x_o\x_l}$ that starts at $\x_o$, let $\x_i, \x_j$ be the intersection of $\ray{\x_o\x_l}$ and $\bdr_i, \bdr_j$ respectively. There is always a segment $\seg{\x_p\x_q}\subset \ray{\x_o\x_l}$, such that if $\x_i,\x_j \in \seg{\x_p\x_q}$ and $\|\x_o\x_i\| < \|\x_o\x_j\|$, then the angular velocities at $\x_i$ and $\x_j$ always satisfy $\omega(\x_i)<\omega(\x_j)$ or $\omega(\x_i)>\omega(\x_j)$.
\end{itemize}

The above properties are satisfied by the non-dissipated dynamics of the general flow models considered in this work. We will develop control schemes based on these properties. 


\subsection{Dissipative dynamics}

In a dissipative flow field, the motion energy of an object in the flow decreases irreversibly, leading the object to arrive at a point where its velocity drops to zero, as in the example in~\eqref{eq:exp1}. In the example given in~\eqref{eq:exp2}, the center point is driven by an external force and the flow spins outward instead of inward, reflecting an inverse model of a dissipated gyre-like flow field. The center point $\x_z$  of the dissipated flow field may overlap with (as in~\eqref{eq:exp2}) or be close to (as in~\eqref{eq:exp1}) the center of the trajectories of non-dissipated system $\x_o$.

In both examples, for any $\x(0) \in \env$, there is always a future time $t$ such that $\x(t)$ finishes a cycle around $\x_z$ and is closer (as in~\eqref{eq:exp1}) or farther (as in~\eqref{eq:exp2}) to $\x_z$ than $\x(0)$. Therefore, we can \textit{roughly} say that the dissipation in the flow field continuously brings an object in the flow towards a trajectory that is closer (farther) to the center. Notice that in~\eqref{eq:exp1}, this pattern is roughly met only when the dissipation is minimal. A large dissipation rate will cause the flow center to deviate significantly from the center of the non-dissipated flow, and the continuous shift of trajectories would not hold.


\subsection{Problem statement}

Obtaining a perfect model of a real-world flow field is generally hard, but rough estimates can be calculated from sparse local measurements or satellite imagery. Since many commonly used gyre-like flow fields have the properties described in Sec.~\ref{sec:nonDissipative}, we assume these properties apply to a wide range of real-world flow fields. Although real-world flows are highly dynamic, seasonal patterns are often persistent enough to leverage for informed control.

We are interested in developing control schemes that leverage these flow fields, are simple enough to be applicable to low-cost or underactuated robots, and do not require a perfect model or knowledge of the flow. The objective is {\it not} to perfectly track a previously designed trajectory, but rather to achieve control over certain performance parameters, \textit{e.g.}, arriving at a region within a time window, circulating in the gyre-like flow over a range of frequencies, etc. 

\textbf{P1:} Assume a flow-field $\env$ satisfying all the properties in Sec.~\ref{sec:nonDissipative}, and that we have:

\begin{itemize}
    \item A robot that can act as a pointable thruster on a time-scale much less than the period of the flow\footnote{On realistic ocean timescales, effectively all robots satisfy this property.}.
    \item Imperfect knowledge about the shape of the gyre orbits.
    \item Imperfect knowledge of the gyre center.
    \item Imperfect knowledge about the robot's distance and orientation relative to the gyre center. 
\end{itemize}

The problem is to design controllers that let the robot traverse a cycle of the gyre-like flow with a period in a given range $(\underline{T},\overline{T})$.


\section{Methodology}\label{sec:methodology}

We leverage the properties of gyre-like flow fields introduced in Sec.~\ref{sec:problem} to design our controller to achieve the goal described in \textbf{P1}. Since non-dissipated trajectories are boundaries of star shapes, the flow field can be mapped to a circular orbital model as in~\eqref{eq:orbit_field}, where $\radius \in \real_+$, $\angpos \in [0,2\pi)$, using a linear map $g:(x_1,x_2)^T\in\env \mapsto (\radius,\angpos)^T\in\orb $, with
\begin{align} \label{eq:orbit_field}
    \orb: \begin{pmatrix}
    \dot{\radius}\\
    \dot{\angpos}
    \end{pmatrix} = \begin{pmatrix}
    -\nu(\radius)\\
    \Omega(\radius,\angpos)
    \end{pmatrix},  
\end{align}
such that $\angpos = \arctan\left ( \frac{x_2-\x_z(2)}{x_1-\x_z(1)} \right )$ (the angle from the gyre center to the robot) and $\radius = \radius_i$ always holds for the same non-dissipative trajectory $\bdr_i$. 

For any $\radius_i$, the trajectory of $\dot{\angpos} = \Omega(\radius_i,\angpos)$ forms a circular orbit $\cir_i$, and $\nu(\radius_i)>0$ captures the continuous shift to an inner circle caused by the dissipation (whereas $\nu(\radius_i)<0$ causes a continuous shift to an outer circle). There is a range of radius $(\underline{\radius},\overline{\radius})$, such that $\forall \angpos_j$, there is always $\frac{d\Omega(\radius,\angpos_j)}{d\radius}<0$, if $\radius\in(\underline{\radius},\overline{\radius})$.  The period of completing one cycle along a given orbit $\cir_i$ is calculated as 
 \begin{align*}
     T_i = \int_{\angpos = 0}^{2\pi} \frac{1}{\Omega(\radius_i,\angpos)} d\angpos,
 \end{align*}
 and therefore $T_i<T_j$ iff. $\underline{\radius}<r_i<r_j<\overline{\radius}$.
 
We design a controller that only relies on the robot's relative location from the center given by
\begin{align} \label{eq:controller}
 \dot{\vec{\radius}} = \control(\vec{\radius})
\end{align}
where $\vec{\radius}$ is a vector along the radial direction and $\|\vec{\radius}\| = \radius$. We consider completing $2\pi~\si{rad}$ around $\x_z$ as one cycle and the trajectory formed over one cycle can be described as $\traj_u = (\radius_u(\angpos), \angpos)^T$, where $\angpos \in [0,2\pi)$ and $\radius_u(\angpos) \in \real_+$.
 
\begin{theorem} \label{thm:periodBound}
Consider a trajectory $\traj_u = (\radius_u(\angpos), \angpos)^T$ formed in a flow field $\orb$, with its dynamics shown in~\eqref{eq:orbit_field} and under a controller as shown in~\eqref{eq:controller}. The time needed to complete a cycle of this trajectory is $T_{\traj_u}$. If $\radius_u(\angpos)\in(\radius_i,\radius_j)$ and $(\radius_i,\radius_j)\subset (\underline{\radius},\overline{\radius})$, then $T_{\traj_u}\in(T_i,T_j)$.
 \end{theorem}
 \begin{proof}
 We have
 \begin{align*}
     T_{\traj_u} = \int_{\angpos = 0}^{2\pi} \frac{1}{\Omega(\radius_u(\angpos),\angpos)} d\angpos.
 \end{align*}
 Given that for any $\angpos$, there is $\Omega(\radius_i,\angpos)>\Omega(\radius_j,\angpos)$ if $\radius_i<\radius_j$ and  $(\radius_i,\radius_j)\subset (\underline{\radius},\overline{\radius})$. Since $\radius_u(\angpos)\in(\radius_i,\radius_j), \forall \angpos$, we can show that
 \begin{align*}
       \int_{\angpos = 0}^{2\pi} \frac{1}{\Omega(\radius_i,\angpos)} d\angpos <T_{\traj_u} <\int_{\angpos = 0}^{2\pi} \frac{1}{\Omega(\radius_j,\angpos)} d\angpos.
 \end{align*}
 Therefore $T_{\traj_u}\in(T_i,T_j)$.
 \end{proof}

We can conclude that in a gyre-like flow, if (1) we apply a controller only in the radial direction (\textit{i.e.} towards or away from the center with respect to the object) and (2) the trajectory formed falls between two non-dissipated boundaries $\bdr_i$ and $\bdr_j$ (without loss of generality, assuming that $\star_i\subset \star_j$), then the time needed to complete a cycle of the trajectory \textit{roughly} falls in $(T_i,T_j)$ where $T_i$ and $T_j$ are the time periods completing one cycle of $\bdr_i$ and $\bdr_j$ respectively. For any given range of periodicity $(\underline{T},\overline{T})$, if we can find two non-dissipated trajectories with periodicities in $(\underline{T},\overline{T})$, then applying control in the radial direction to keep the robot in between the two boundary trajectories will ensure its periodicity is bounded by $(\underline{T},\overline{T})$.

Consider a robot capable of producing thrust $u \in [0, u_{max}]$ along an arbitrary orientation $\theta_r\in [0,2\pi)$. Most underactuated robots can be modeled as such given a sufficient time scale $\tau \ll \underline{T}$. Given a desired range of periodicity $(\underline{T},\overline{T})$ and corresponding radii $(\underline{r},\overline{r}) \in \orb$, $\underline{r} < \overline{r}$, the bang-bang control strategy given by
\begin{equation}\label{eq:controllerSpecific}
    (u,\theta_r) = \begin{cases} 
        (u_{max}, \hphantom{-}\angpos) & r \leq \underline{r} \\
        (0, \angpos) & r \in (\underline{r},\overline{r}) \\
        (u_{max}, -\angpos) & r \geq \overline{r}
    \end{cases}
\end{equation}
will ensure the robot stays within the desired region, and therefore maintains the desired periodicity, as long as the region is sufficiently wide $\overline{r} - \underline{r} \geq \epsilon$. 


\section{Experimental Setup} \label{sec:expPlatform}

We experimentally validate the proposed strategy given by \eqref{eq:controllerSpecific} on the Modboat, which is shown in the inset to Fig.~\ref{fig:tankProps} and described in~\cite{modboatsOnline}.

\begin{figure}[t]
    \centering
    \includegraphics[width=\linewidth]{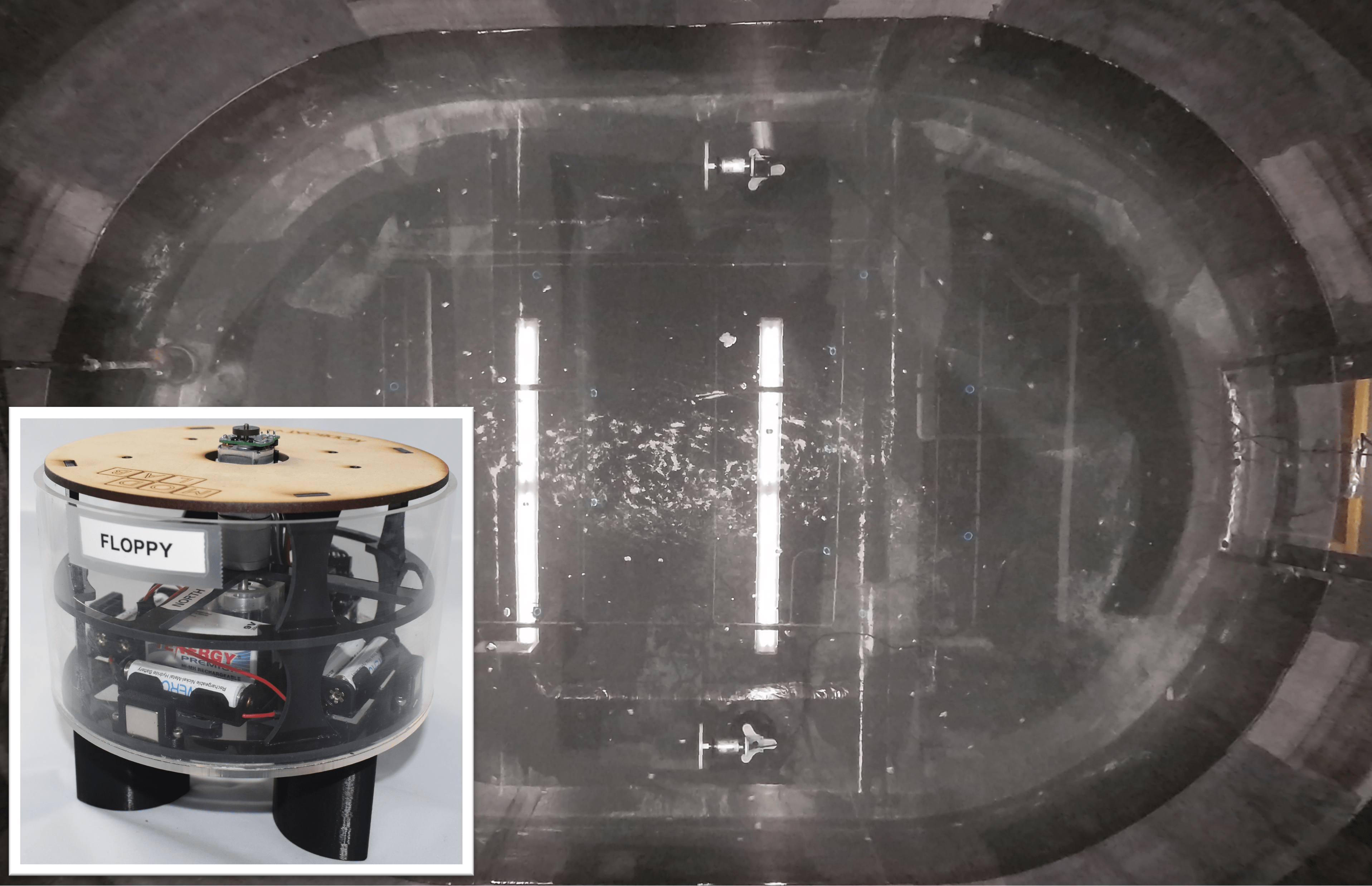}
    \caption{An overhead view of the racetrack shaped experimental tank. The two propellers are visible at the top and bottom of the image, and a glass observation window is present on the right-hand side. Inset: a photo of the Modboat platform used for the experiments in this work. The Modboat is described in~\cite{modboatsOnline}, and the controller used in this work is developed in~\cite{Knizhnik2021ThrustRobot}.}
    \label{fig:tankProps}
\end{figure}


\subsection{Robotic Platform}

The Modboat is a low-cost, underactuated, surface swimming robot developed by the authors at the University of Pennsylvania~\cite{modboatsOnline}\cite{Knizhnik2020}. It consists of a light lower body with two passive flippers, connected by a motor to a massive upper body. Conservation of momentum allows the bodies to rotate relative to one another when the motor is actuated; the leading flipper swings open until it is limited by a hard stop, at which point it creates thrust for the robot. Oscillating the motor causes both flippers to open in sequence and generates a paddling motion that can be used for propulsion and steering\footnote{An informative video of this motion and details of the robot design can be viewed in~\cite{modboatsOnline}.}. An ESP32 microcontroller on-board runs position control on the angle of the motor shaft, and various control modes are possible by passing waveform parameters to the boat over WiFi. 

In this work, the Modboat uses the forced pendulum controller given by
\begin{equation}\label{eq:contLimCyc}
    u = -K \sin{(\omega t)} - \beta \sin{\left (\theta_r(t) - \theta(t) \right )}
\end{equation}
where $u$ is the torque applied to the motor, $\theta$ is the observed orientation of the Modboat, and $\theta_r = \pm \angpos$ is the reference heading from~\eqref{eq:controllerSpecific} along which thrust is generated.  This control strategy was developed in prior work~\cite{Knizhnik2021ThrustRobot} and allows the Modboat to function as a pointable thruster when averaged over a cycle of length $\tau = 2\pi/\omega$. Since $\tau \ll \underline{T}$ for most realistic gyre like flows, this meets the criteria for~\eqref{eq:controllerSpecific}. When inactive, a holding torque maintans the motor position.

The control gain $K$ controls the amplitude of oscillation, and the angular frequency $\omega$ determines their frequency while $\beta = \omega^2$ is used for resonance~\cite{Knizhnik2021ThrustRobot}. Although the values of $K$ and $\omega$ can be adjusted to change controller performance, $K=15~\si{Nm}$ and $\omega=2\pi~\si{rad/s}$ are used for all experimental validation in this work. Under these parameters, the Modboat produces $21\pm1.7~\si{mN}$ of thrust on average along the direction given by $\theta_r$, which results in a top speed of $9.3\pm0.37~\si{cm/s}$ in still water.

We note that in prior work the control law~\eqref{eq:contLimCyc} is combined with a thrust direction controller to allow waypoint tracking~\cite{Knizhnik2021ThrustRobot}. This full control architecture is used when testing in still water and referred to as the \textit{naive/waypoint approach} in this work. In flow-based control we are interested in the effects produced by radial {\it thrust} rather than radial \textit{travel}. As such, only the lowest level of control~\eqref{eq:contLimCyc} is used in the gyre flow experiments.


\subsection{Flow Field Generation} \label{sec:flowGeneration}

Experiments involving the Modboat are conducted in a $4.5~\si{m} \times 3.0~\si{m} \times 1.2~\si{m}$ racetrack shaped tank of water equipped with an OptiTrack motion capture system that provides planar position, orientation, and velocity data at up to $120~\si{Hz}$. Flows of various shapes and speeds can be generated in this tank by the use of propellers. For this work, a single gyre was created by placing two propellers, mounted horizontally and spinning at $200~\si{rpm}$, along the straight edges of the tank as seen in Fig.~\ref{fig:tankProps}. In this configuration, the flow is shaped by the walls of the tank and forms an outward spiraling gyre. Dissipative radial flow is observed to be small, so the flow is mostly angular and meets the criteria described in Sec.~\ref{sec:nonDissipative}.

The flow in the tank is measured using a Vectrino acoustic velocimeter~\cite{nortek}, which provides two-dimensional velocity measurements of flow seeded with $50~\si{\mu m}$ glass beads. The flow is initiated and allowed to reach steady state. Assuming the flow is two-dimensional, the Vectrino is used to take velocity measurements at various locations, which can then be used to fit the model described in Sec.~\ref{sec:flow_field_model}. 


\subsection{Flow Field Modeling} \label{sec:flow_field_model}

Modeling the flow described in Sec.~\ref{sec:flowGeneration} is challenging, since the experimental enclosure is shaped like a racetrack (see Fig.~\ref{fig:tankProps}), which cannot be easily modeled by a continuous function. Moreover, the flow is driven from the external edge.  To the authors' best knowledge, these flows have not been studied in depth. However, we experimentally observe that the flow forms mostly concentric orbits whose shape is given by the shape of the enclosure, and we assume the flows roughly meet the properties listed in Sec.~\ref{sec:problem}.

An approximation of the tank shape can be generated using interpolated implicit functions~\cite{Turk2001ImplicitInterpolate,Chaimowicz2005ControllingFunctions,Hsieh2007StabilizationSensing}, which can be used to generate continuous approximations $\gamma = s(x,y)$ of arbitrary functions. It has been shown that for such curves a shape navigation function 
\begin{equation}\label{eq:shapeNavigation}
    \phi(q) = \frac{\gamma^2}{\gamma^2 + (R_0 - \norm{q})}
\end{equation}
can be defined~\cite{Hsieh2007StabilizationSensing} for a robot located at $q\in\mathbb{R}^2$ in a circular workspace of radius $R_0$.  The gradient of the shape navigation function, $\nabla_q \phi(q)$, gives a potential field
\begin{equation}\label{eq:potentialField}
    f = -K_r \nabla \phi - K_\theta \nabla \times \begin{bmatrix} 0 & 0 & \gamma \end{bmatrix}^T  
\end{equation}
in which $K_r$ and $K_\theta$ are constants balancing the radial and angular components of the field, respectively. This field approximates concentric orbits flowing towards the boundary $\gamma$~\cite{Hsieh2007StabilizationSensing}, and $K_r$ and $K_\theta$ can be obtained by a least-squares fit using the flow velocity measurements obtained in Sec.~\ref{sec:flowGeneration}.

The flow-field in~\eqref{eq:potentialField} is used for simulation studies, but the controller~\eqref{eq:controllerSpecific} requires a radius relative to a circular flow. To achieve this we transform the racetrack coordinate system into the form shown in~\eqref{eq:orbit_field} using 
\begin{equation}\label{eq:circleMap}
    C(q) = r_{max}\sqrt{1 + \gamma} \left (\frac{q - g}{\norm{q - g}} \right )
\end{equation}
for a gyre centered at $g\in\mathbb{R}^2$~\cite{Rimon1992ExactFunctions}.  It is important to note that this flow model is only approximate. In~\cite{Hsieh2007StabilizationSensing}, the primary desired behavior is approach and stabilization on the boundary $\gamma$, so the exact behavior of the function in the interior is not significant. In flow characterization, however, this is far more significant. In fact, the exact interior behavior of the implicit interpolating function $\gamma$ shows significant dependence on the exact interior constraints, which impacts the shape of the flow.  While these issues require further exploration, they are beyond the scope of this work.


\section{Results} \label{sec:results}


\subsection{Simulation}\label{sec:resultsSim}

We use the flow field model given in Sec.~\ref{sec:flow_field_model} to demonstrate that the proposed high-level controller given by~\eqref{eq:controller} can effectively enable a minimally actuated robot to periodically cycle in gyre-like flows such that the period lies within a desired range. We note that the flow field is approximate in two respects. First, it fails to capture the shape of the flow close to the wall of the tank, where features like the observation window disturb it. Second, although the real flow shows monotonically changing angular velocities the simulated flow field does so only in the center, as shown in Fig.~\ref{fig:angularVelocity}. Therefore, in simulation we consider a smaller interior region where the constraints from Theorem~\ref{thm:periodBound} hold (shown in Fig.~\ref{fig:boundaryInPool}), whereas experimental validation utilizes the entire area.

We consider two non-dissipated boundary trajectories within this interior region, as shown in Fig.~\ref{fig:pool_model}. Without a radial component of the flow, completing one cycle of the outer bound takes $56.92~\si{s}$ and the period of the inner bound is $42.67~\si{s}$. A simplified version of~\eqref{eq:controllerSpecific} is used in simulation; the controller activates whenever the robot reaches the outer bound, and remains activated until the inner bound is reached. No outward control is applied because the radial component of the flow provides outward movement. While active, the controller provides a velocity randomly sampled from the range $(0.01,0.04)~\si{m/s}$.

\begin{figure}[t]
    \centering
    \subfloat[\label{fig:boundaryInPool}]{\includegraphics[width=0.47\linewidth, trim={0.0cm 0 1.0cm 0.0cm},clip]{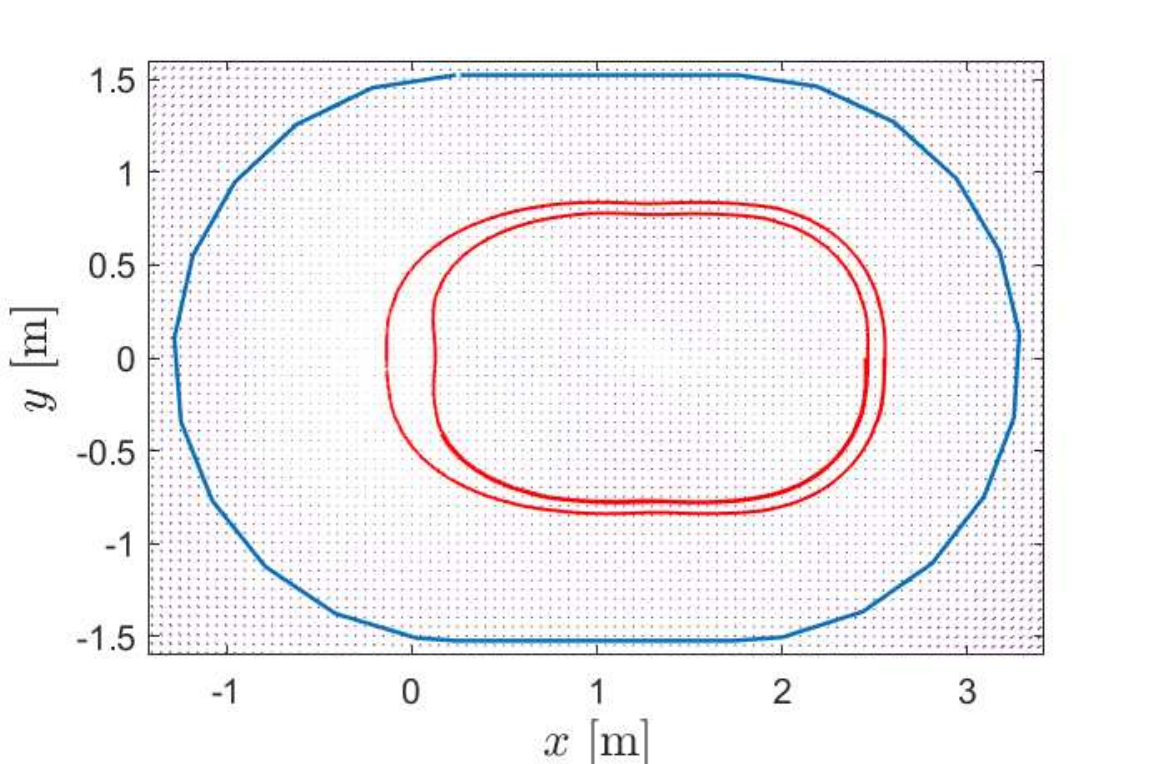}
       }
          \hfill
    \subfloat[\label{fig:angularVelocity}]{\includegraphics[width=0.46\linewidth, trim={0.0cm 0 0.5cm 2.5cm},clip]{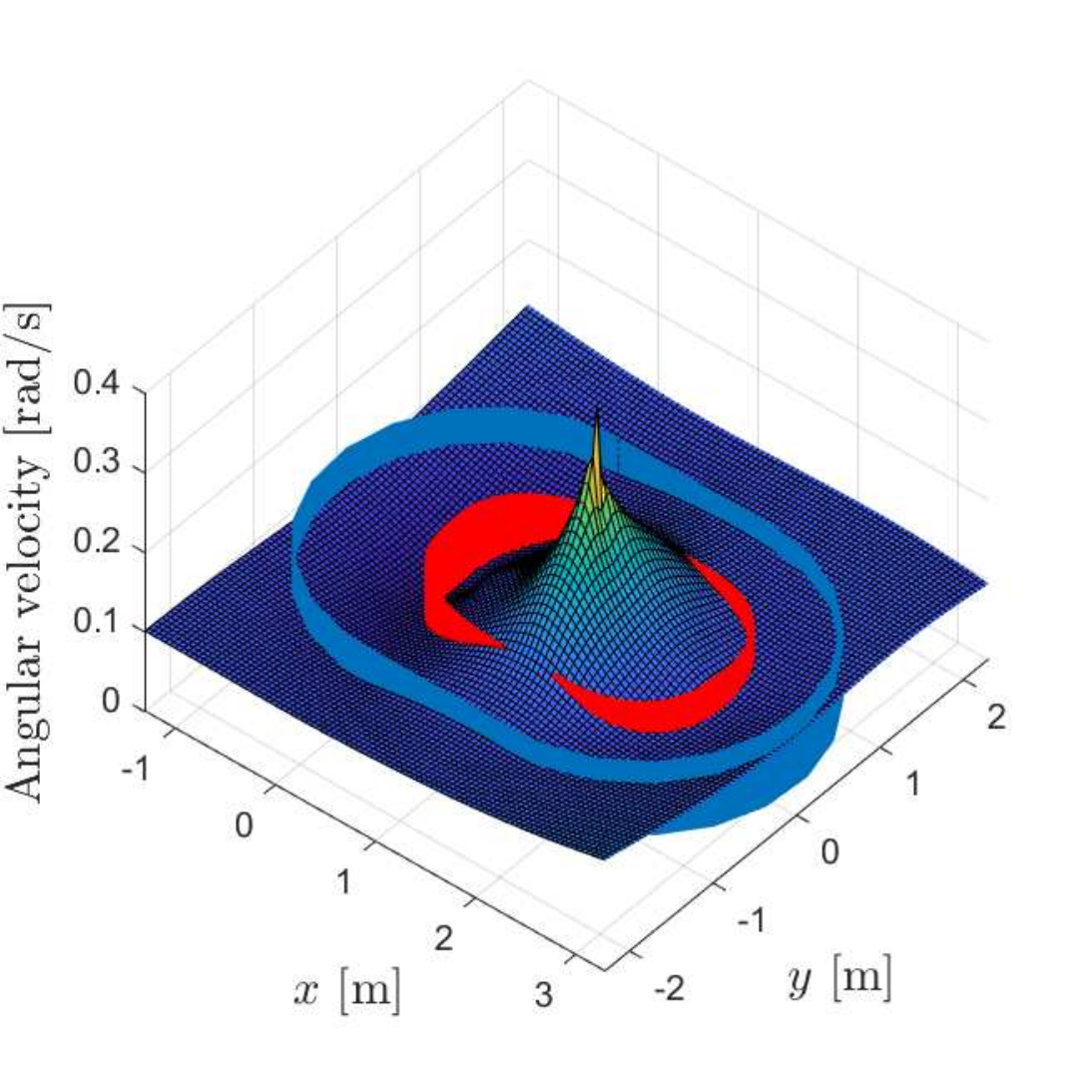}
       }
    \caption{Regions relevant to the simulation studies, with an overhead view in (a) and angular velocity plotted in (b). The tank boundary is shown in blue, and the two non-dissipated boundary trajectories used in the study are shown in red. Note the non-monotonic behavior outside the boundary region in (b).}
    \label{fig:pool_model}\vspace{-2mm}
\end{figure}

In Fig.~\ref{fig:sim_results} we show the trajectories of the robots in a $5000~\si{s}$ simulation as well as the periods the robot takes to complete each cycle. Although the trajectories of each round vary significantly, the robot always completes a cycle with a period that falls in the specified range of $(42.67, 56.92)~\si{s}$.

\begin{figure}[t]
    \centering
    \subfloat[\label{fig:simTraj}]{\includegraphics[width=0.47\linewidth, trim={0.25cm 0.0cm 1.0cm 0.5cm},clip]{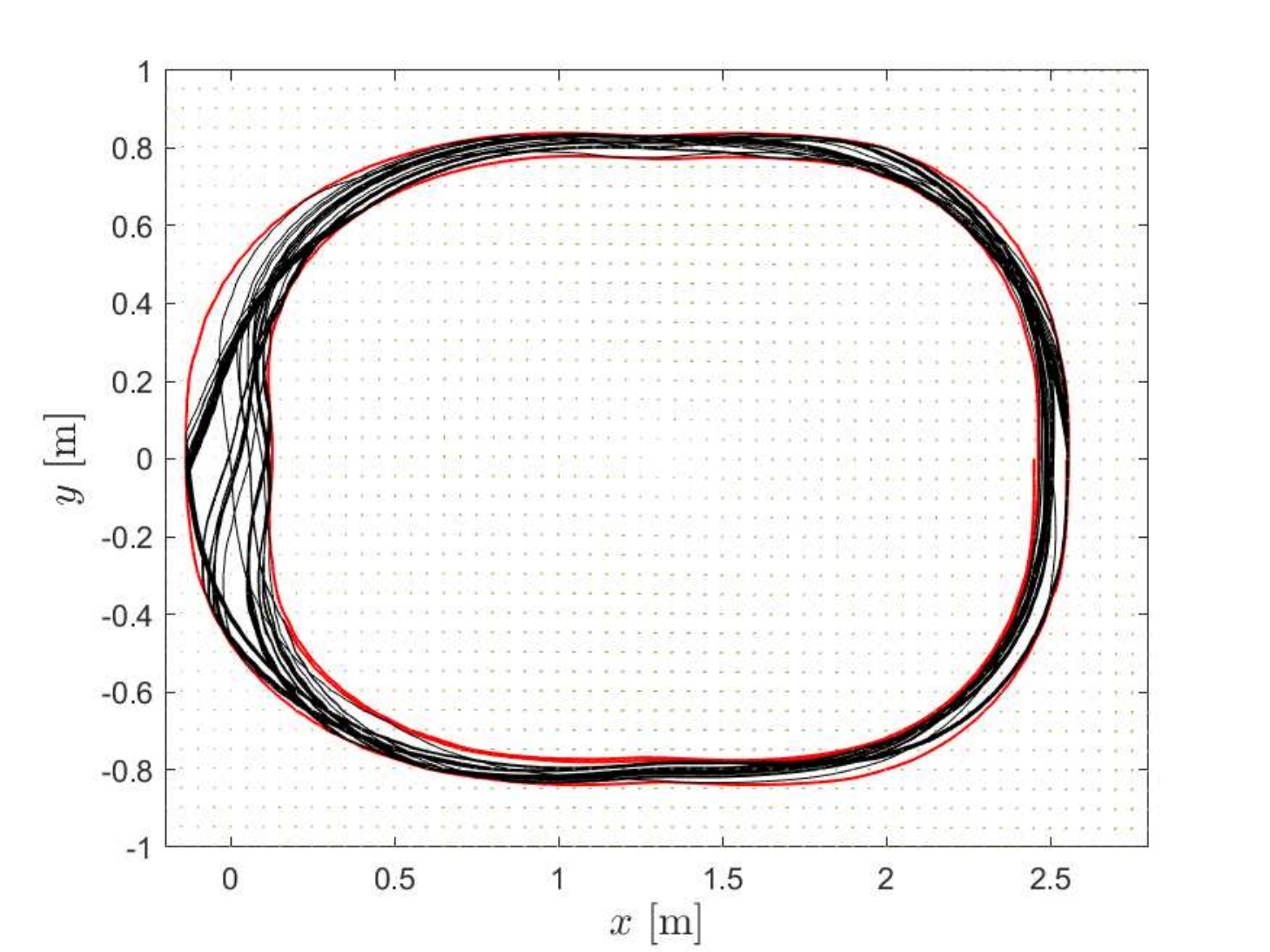}
       }
          \hfill
    \subfloat[\label{fig:simOrbitTimes}]{\includegraphics[width=0.47\linewidth, trim={0.25cm 0.0cm 1.0cm 0.5cm},clip]{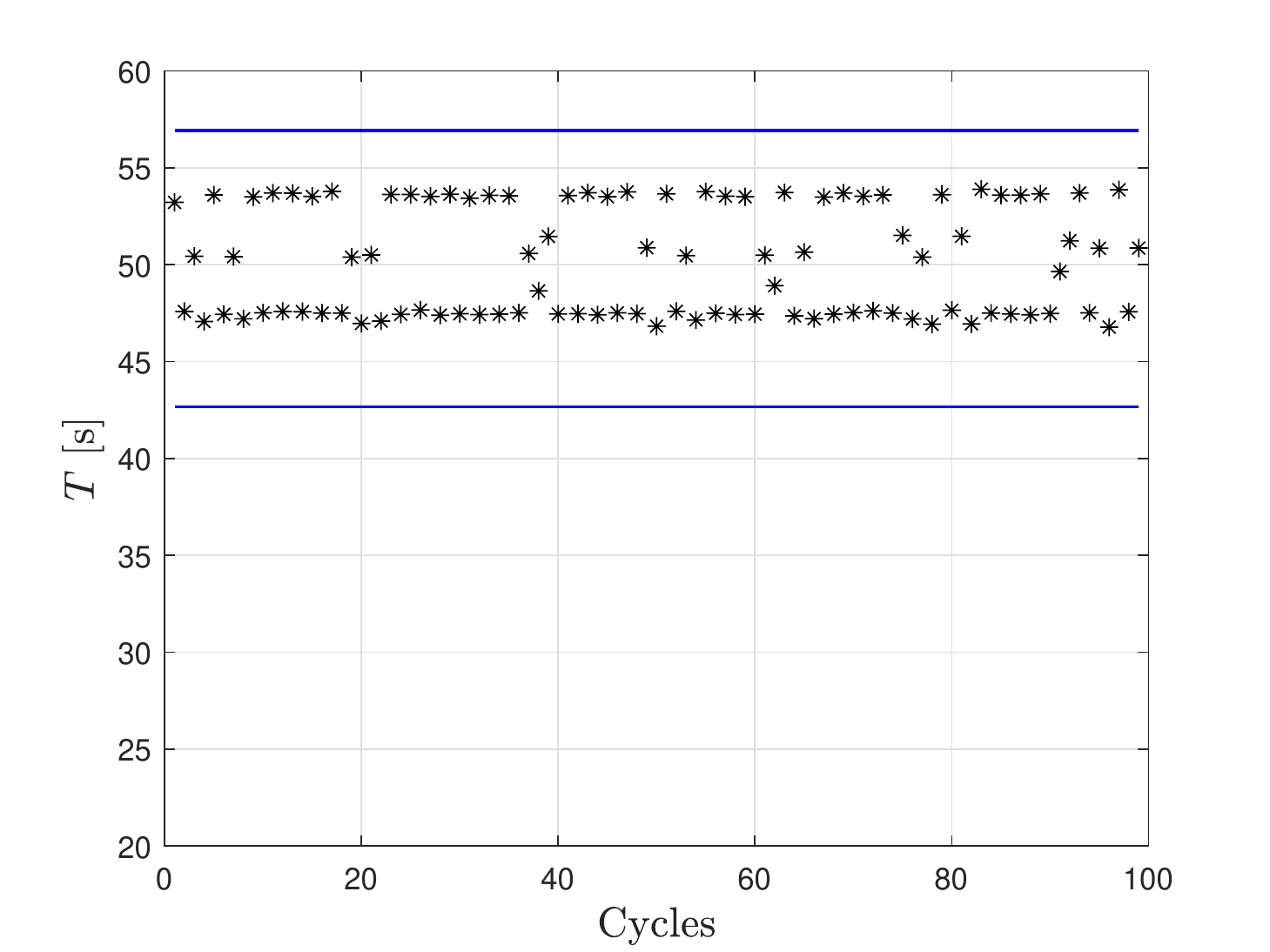}
       }
    \caption{Simulation study results for $5000~\si{s}$. (a) The simulated trajectory and the non-dissipative boundary trajectories are shown in black and red, respectively. (b) The periods $T$ for each cycle. Every cycle is completed between $42.67~\si{s}$ and $56.92~\si{s}$.}
    \label{fig:sim_results}\vspace{-2mm}
\end{figure}


\subsection{Experiments}\label{sec:resultsExp}

Poor agreement between the flow model (as presented in Sec.~\ref{sec:flow_field_model} and Fig.~\ref{fig:angularVelocity}) and the observed flow made choosing appropriate non-dissipated boundary trajectories for experimentation difficult. As an approximation, boundaries were generated by using the inverse of the transformation in~\eqref{eq:circleMap} to generate racetrack trajectories corresponding to various radii spaced $0.1~\si{m}$ apart. The closest flow measurements from Sec.~\ref{sec:flowGeneration} were assumed to hold constant along the trajectory, and the desired periodicity was then approximated as the path length divided by the velocity. Controller~\eqref{eq:controllerSpecific} was then used to enable the Modboat to stay within the desired region. A sample such trajectory is shown in Fig.~\ref{fig:trajectory}, and the accompanying control variables in Fig.~\ref{fig:controlVars}. The periodicity of the resulting orbits is presented in Table~\ref{tab:resultsPerRadius} along with the corresponding control effort (measured as active time normalized to orbital period); the listed radius corresponds to the interior boundary of the desired region. Each test was repeated five times, and the presented data is averaged over all recorded complete cycles.

\begin{table}[t]
    \centering
    \caption{Orbital period \& control effort per period for flow-based control at interior radii for $0.1\si{m}$ wide bands, as $\mu\pm\sigma$.}
    \begin{tabular}{rl|c|c|c}
    \toprule
    Radius & $[\si{m}]$ & $0.90$ & $1.1$ & $1.3$ \\ \midrule
    Prescribed Period & $[\si{s}]$ & $82-97$ & $75-83$ & $66-70$ \\
    Observed Period & $[\si{s}]$ & $98 \pm 11 $ & $85 \pm 5.5$ & $76 \pm 3.6$ \\
    Control Effort & $[\si{\%}]$ & $10 \pm 3.7$ & $16 \pm 5.5$ & $23 \pm 9.1$\\\bottomrule
    \end{tabular}
    \label{tab:resultsPerRadius}
\end{table}
\begin{table}[t]
    \centering
    \caption{Orbital periods \& control effort per period for varying control modes in $r\in(1.3-1.4)$. Presented as $\mu \pm \sigma$.}
    \begin{tabular}{rl|c|c|c|c}
    \toprule
     Control & & \multicolumn{2}{c}{Waypoints} & No control & Flow-based \\
    strategy &  & w/out flow & w/ flow & w/ flow & w/ flow \\ \midrule
    Period & $[\si{s}]$ & $130 \pm 4.0 $ & $54 \pm 1.5$ & $72 \pm 5.5$ & $76 \pm 3.6$ \\
    Control & $[\si{s}]$ & $130 \pm 4.0 $ & $54 \pm 1.5$ & $0.0 \pm 0.0$ & $17 \pm 6.2$ \\\bottomrule \vspace{-6ex}
    \end{tabular}
    \label{tab:resultsPerController}
\end{table}

Since flow-based control is meant to aid resource-constrained robots and create more energy efficient swimming, we also compared the observed performance to a naive swimming approach designed for still water using the desaturated thrust direction controller developed in~\cite{Knizhnik2021ThrustRobot}, which aims the Modboat at successive waypoints. The orbit was approximated by waypoints placed every $0.5~\si{m}$ along the transformed radius $r=1.35~\si{m}$ and executed both with and without flow. This was compared to the flow-based control maintaining the same orbit, and to a case where no control was applied\footnote{To set the initial region, the Modboat is powered according to~\eqref{eq:controllerSpecific} until it gets to the desired region. It is then unpowered for the rest of the test.}. Each test was repeated five times, and the results are shown in Table~\ref{tab:resultsPerController}, with control effort given as active time.

\begin{figure}[t] 
    \centering
    \includegraphics[width=\linewidth, trim={1.0cm 0.5cm 1.0cm 0.5cm}]{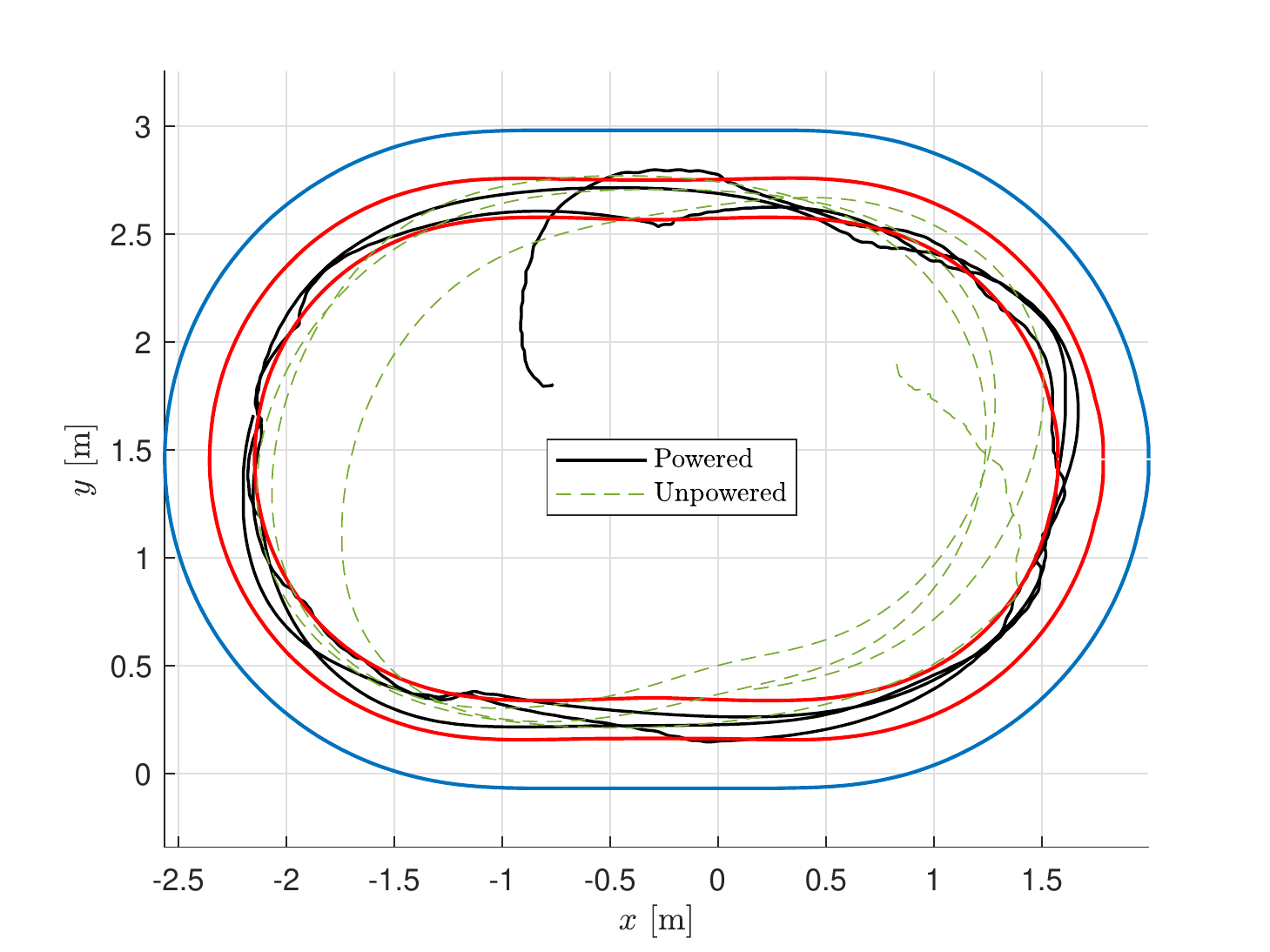}
    \caption{Sample flow-based control (solid black) and unpowered (dashed green) trajectories of a Modboat in a clockwise flow, starting from the interior. The desired region boundaries are shown in red, and the tank boundary in blue. Note the difference in shape between the observed orbits and desired region, especially at the bottom right. }
    \label{fig:trajectory}
\end{figure}\vspace{-3ex}


\section{Discussion}\label{sec:discussion}

The simulation results validate our conclusion about the controller designed in Sec.~\ref{sec:methodology}. In the simulation presented in Fig.~\ref{fig:sim_results}, the controller is able to keep the robot between the two boundary oribts using only a minimal radial velocity. Moreover, the controller succeeds despite providing inaccurate output velocities and knowing only its relative position in the flow. Lacking a detailed model of the flow field or knowledge of its own velocity in the flow, the robot can still maintain a semi-periodic orbiting motion in the flow, even though the actual trajectory varies every round. The period of each cycle varies as well but always stays between the periods of the two non-dissipative boundary orbits. This aligns with our predictions as in Theorem~\ref{thm:periodBound}. However, the simulation is run for a region that is significantly smaller than the region used in real experiments (compare Figs.~\ref{fig:boundaryInPool} and~\ref{fig:trajectory}). This is necessary to stay within a region of monotonically changing angular velocity, which holds in real world conditions but not in the model, as can be seen in Fig.~\ref{fig:pool_model}.

Experimental results similarly validate our expectation that flow-based control can maintain consistent orbital periods by leveraging approximate knowledge of the gyre center and orbit shape. As presented in Table~\ref{tab:resultsPerRadius}, flow-based control consistently generates orbital periods that are \textit{close to} the desired periodicity but higher than it. This is reasonable, since the desired values assume massless fluid particles and the Modboat is expected to orbit more slowly. This also results in occasionally exiting the desired region (as shown in Fig.~\ref{fig:trajectory}), since it takes the controller~\eqref{eq:controllerSpecific} time to overcome the Modboat's inertia. This further pushes the observed orbital periods past their target ranges. 

The observed periods deviate most strongly from the desired values at the largest radii. We believe this is because close to the tank wall the actual flow field deviates more significantly from the assumed flow of the model, especially around the viewing window seen in the right of Fig.~\ref{fig:tankProps}. This is further corroborated by the unpowered trajectory in Fig.~\ref{fig:trajectory}, which shows orbits whose shape does not match the desired boundaries, especially in the bottom right (which follows the observation window in the flow direction). A better flow model can be constructed and is expected to lead to more accurate boundary generation and better agreement between the desired and experimentally achieved period bounds.

Even with poor boundary generation, however, experimental validation demonstrates that our flow-based control is able to keep the Modboat mostly in the desired region. Our strategy provides clear benefits in terms of energy efficiency when compared to more naive approaches, as summarized in Table~\ref{tab:resultsPerController}; it takes $1.4$ times as long to complete an orbit of the gyre, but only requires $31\%$ as much control effort. This is clear in Fig.~\ref{fig:controlVars}, where the controller is visibly inactive for the majority of the test. Thus a resource-constrained robot utilizing our control approach would last nearly $3.2$ times as long (in the worst case) in an ocean monitoring situation (assuming non-propulsive power consumption, i.e. hotel load is minimal).

\begin{figure}[t]
    \centering
    \includegraphics[width=\linewidth, trim={0.5cm 0.5cm 0.5cm 0.0cm}]{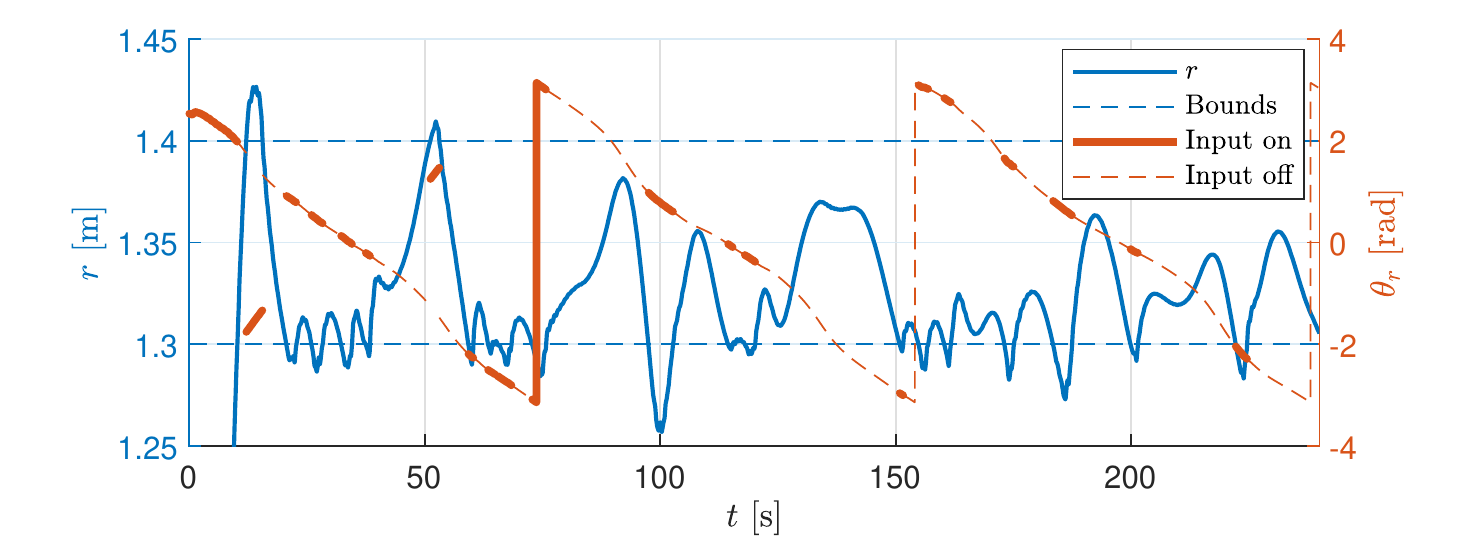}
    \caption{Control variables for the powered test in Fig.~\ref{fig:trajectory}. The blue line indicates the transformed radius $r$, bounded by dotted lines for $(\underline{r},\overline{r})$. The thick (thin) orange line indicates $\theta_r$ from~\eqref{eq:controller} when the controller is on (off).}
    \label{fig:controlVars}
\end{figure}


\section{Conclusion} \label{sec:conclusion}
We have shown that a \textit{flow-based control approach} using approximate knowledge of the gyre center is sufficient to allow simple pointable thruster robots --- with the Modboat as an example --- to maintain a desired periodicity within a flow field. This allows them to function as active drifters, maintaining a desired orbital region with minimal control effort while mostly utilizing surrounding flow for navigation and control. A simple control scheme compatible with highly uncertain and time-averaged control inputs was used, which makes these results applicable to a wide range of underactuated or low-powered systems.

We have also demonstrated that this flow-based approach shows significant energy savings over naive waypoint tracking approaches. Observed energy savings are over $68\si{\%}$ despite poor correspondence between the selected orbit shape and any true orbit. We anticipate that better flow modeling and the selection of a more appropriate region can lead to even greater energy efficiency, as can more responsive control laws. This opens the door to long monitoring missions that are not significantly constrained by the energy density of their batteries even for resource-constrained systems. 

Since poor flow modeling resulted in poor boundary region selection in experiments and limited useful area in simulation, future work will consider more accurate flow modeling to improve the performance of the system, even with limited information. Future work will also consider whether flow modeling can be avoided by using local measurements to approximate the necessary flow knowledge. Furthermore, in this work, we focused simulation and experimental validation in near perfect gyre flows. Another direction of future work including validating and extending the proposed strategy to more general flows.


\bibliographystyle{./bibliography/IEEEtran}
\bibliography{./bibliography/IEEEabrv,./bibliography/nonpaper,./bibliography/references,./bibliography/add,./bibliography/field_planning,./bibliography/LCSrefs,./bibliography/stochasticEscape,./bibliography/drifter}

\end{document}